\documentclass[letterpaper, 10 pt, conference]{ieeeconf}  

\IEEEoverridecommandlockouts                              
\usepackage{times} 
\usepackage{amsmath,amsfonts} 
\usepackage{amssymb}  
\usepackage{tabularx}
\usepackage[normalem]{ulem}
\usepackage{dsfont}
\usepackage{graphicx}
\usepackage{subcaption}
\usepackage[dvipsnames,usenames]{xcolor}
\usepackage[colorlinks,%
linkcolor=BrickRed,%
filecolor=BrickRed,%
citecolor=RoyalPurple,%
]{hyperref}
\usepackage{algorithm}
\usepackage{algorithmic}
\usepackage{overpic}
\usepackage{float}

\newcommand{\at}[1]{{\color{black} #1}}

\def\notes#1{\marginpar{\tiny #1}\typeout{Notes!
Notes!
Notes!
}}
\renewcommand{\notes}[1]{\typeout{notes!}}

\newtheorem{theorem}{Theorem}

\newtheorem{definition}{Definition}

\newtheorem{remark}{Remark}
\newtheorem{proposition}{Proposition}

\newtheorem{assumption}{Assumption}

\def\beq{\begin{eqnarray}} 
\def\bc{\begin{center}} 
\def\be{\begin{enumerate}}
\def\bi{\begin{itemize}} 
\def\bs{\begin{small}}
\def\bS{\begin{slide}}
\def\ec{\end{center}} 
\def\ee{\end{enumerate}}
\def\ei{\end{itemize}}
\def\es{\end{small}}
\def\eS{\end{slide}}
\def\eeq{\end{eqnarray}}

\newcounter{rmnum}

\newcounter{anum}

\title{\LARGE \bf
Conditional Sampling via Wasserstein Autoencoders \\and Triangular Transport
}
\author{Mohammad Al-Jarrah$^{\star,\dagger}$  \and Michele Martino $^\dagger$ \and Marcus Yim$^\ddagger$\and 
Bamdad Hosseini$^\dagger$ \and 
Amirhossein Taghvaei$^\star$
    {\thanks{$^\star$Department of Aeronautics \& Astronautics, University of Washington, Seattle {\tt\small mohd9485@uw.edu,amirtag@uw.edu}.}}
    {\thanks{$^\dagger$Department of Applied Mathematics, University of Washington, Seattle
        {\tt\small mohd9485@uw.edu,michem29@uw.edu,bamdadh@uw.edu}.}}
    {\thanks{$^\ddagger$Department of Mathematics, University of Washington, Seattle {\tt\small myim3@uw.edu}.}}
        \thanks{
        M. Al-Jarrah led the computational implementation. M. Martino led the methodology development. M. Yim developed the flow field experiment. B. Hosseini and A. Taghvaei supervised the research and contributed to the preparation of the manuscript.
        M. Al-Jarrah and A. Taghvaei are supported by the National Science Foundation (NSF) award EPCN-2318977. B. Hosseini and M. Martino are supported by the NSF award DMS-2337678. \copyright~2026 IEEE.
        }
}

\begin{document}
	
      \maketitle
	 \thispagestyle{empty}
	 \pagestyle{empty}

\begin{abstract}
We present Conditional Wasserstein Autoencoders (CWAEs), a framework for conditional simulation that exploits low-dimensional structure in both the conditioned and the conditioning variables. The key idea is to modify a Wasserstein autoencoder to use a (block-) triangular decoder and impose an appropriate independence assumption on the latent variables. We show that the resulting model gives an autoencoder that can exploit low-dimensional structure while simultaneously the decoder can be used for conditional simulation. We explore various theoretical properties of CWAEs, including their connections to conditional optimal transport (OT) problems. We also present alternative formulations that lead to three architectural variants forming the foundation of our algorithms. We present a series of numerical experiments that demonstrate that our different CWAE variants achieve substantial reductions in approximation error relative to the low-rank ensemble Kalman filter (LREnKF), particularly in problems where the support of the conditional measures is truly low-dimensional.

\end{abstract}


\section{Introduction}
Sampling from conditional probability distributions is a fundamental task in Bayesian inference, nonlinear filtering, and machine learning~
\cite{baptista2024conditional,kaipio2005statistical,marzouk2017sampling,stuart2010inverse}. Given samples from a joint distribution $P_{Y,X}$, the objective is to
generate samples from the conditional distribution $P_{X|Y=y}$ for
arbitrary values of the conditioning variable $y$. In nonlinear filtering, for instance, $X$ represents unknown states of a system and $Y$ denotes observed noisy sensory data, so the posterior $P_{X|Y=y}$ fully characterizes the inferred quantities and their uncertainties. This task becomes  challenging when $X$ and $Y$ are high-dimensional, for instance, when they represent physical fields or images, as classical methods such as particle filters suffer from the curse of dimensionality~\cite{bengtsson08,beskos2014error,bickel2008sharp,rebeschini2015can}. The goal of this work is to develop a scalable, data-driven framework for conditional sampling that exploits low-dimensional structure whenever it is present.

Our approach is inspired by two complementary methodologies: \emph{measure transport} methods for conditional simulation, and autoencoder-based generative models. In the transport framework, one seeks a deterministic map, typically  block-triangular, that transforms a reference distribution into the joint distribution $P_{Y,X}$, from which the conditional distributions $P_{X|Y=y}$ can be accessed directly after training
~\cite{marzouk2017sampling,baptista2024conditional,hosseini2025conditional,al2025error,taghvaei2021optimal}. These transport maps have provided an effective alternative to particle filters and the ensemble Kalman filter (EnKF)~\cite{al-jarrah2024nonlinear,al2024data,al2025fast}.
However, their application becomes increasingly difficult as the ambient dimension grows. Existing high-dimensional methods, such as localized EnKF or particle filter variants, address this by incorporating problem-specific low-dimensional structure on a case-by-case basis~\cite{houtekamer2001sequential,snyder2008obstacles,rebeschini2015can}. Our goal, by contrast, is to develop a framework that discovers and exploits such structure automatically from data.

A parallel line of work pursues this goal through explicit dimensionality reduction. The likelihood-informed subspace (LIS) approach identifies directions along which the posterior departs most significantly from the prior using likelihood gradients, enabling inference in a reduced subspace~\cite{cui2014likelihood,cui2021analysis}. More broadly, low-rank and subspace-based methods have been developed for filtering and inverse problems~\cite{leprovost2022lowrank}, and recent work has combined these ideas with triangular transport maps~\cite{cui2025subspace}. While effective, these approaches typically require access to likelihood gradients that may not be available in a purely data-driven setting. Here, we take a different approach and propose to learn the relevant low-dimensional structure directly from data.

To achieve this, we draw on the autoencoder framework~\cite{kingma2014autoencoding}, and in particular the Wasserstein autoencoder (WAE)~\cite{tolstikhin2018wasserstein}. The key idea is to introduce a low-dimensional latent random variable that captures the essential structure of the data, and to train an encoder-decoder pair by minimizing the Wasserstein distance between the generated and true distributions, subject to a constraint on the aggregated latent distribution. We generalize this to the conditional setting by combining the WAE framework with block-triangular transport. Specifically, we show that imposing a block-triangular structure on the decoder, together with independently coupled latent variables, yields a conditional generative model whose training objective admits an interpretation as a WAE on the joint space. Furthermore, by imposing a block-triangular structure over the class of admissible
encoders, we formulate the \emph{Conditional Wasserstein Autoencoder (CWAE)} which directly  minimizes a specific conditional transport objective. This provides a scalable, fully data-driven alternative to existing transport and subspace-based methods. 
Notably, our approach differs from the Paired Wasserstein Autoencoder for conditional sampling~\cite{piening2024paired}: rather than imposing a special paired WAE structure, we adopt a more general block-triangular latent structure, and we minimize a conditional transport cost that directly measures the discrepancy between conditional distributions, as opposed to comparing joint distributions. This distinction is discussed in detail in Sec.~\ref{sec:CWAE}.
We analyze several equivalent formulations of our method, which agree at optimality but differ in behavior when suboptimal, and investigate the associated architectural design choices on a range of numerical experiments.

The rest of the paper is organized as follows: 
Sec.~\ref{sec:problem_formulation} includes the problem formulation and the background on WAE; Sec.~\ref{sec:CWAE} contains the proposed methodology; and Sec.~\ref{sec:numerics} presents the numerical algorithm accompanied by several numerical experiments.

\section{Problem Formulation \& Background}\label{sec:problem_formulation}
\subsection{Problem Setup}\label{sec: problem setup}
The training data consists of paired samples from two random  variables $(Y, X)$, whose joint distribution is denoted $P_{Y,X}$. The  variable $Y \in \mathcal{Y}$ plays the role of an observation, while $X \in \mathcal{X}$ represents  the quantity of interest. Both $\mathcal{X}$ and $\mathcal{Y}$ are high-dimensional ambient spaces. 
Our goal is to learn, from data alone, an efficient sampler for the conditional distribution $P_{X|Y=y}$ that works for any query value  $y \in \mathcal{Y}$. Directly learning conditional distributions in high-dimensional spaces is computationally challenging.  The key assumption that makes this problem tractable is that the joint distribution $P_{Y,X}$ admits a  \emph{low-dimensional structure}: although $X$ and $Y$ live in high-dimensional spaces, their statistical dependence is driven by a much smaller number of \emph{degrees of freedom}.

To formalize this, we introduce two independent latent random variables  $Z \in \mathcal{Z}$ and $U \in \mathcal{U}$, where $\mathcal{Z}$ and $\mathcal{U}$ are low-dimensional latent spaces.  Their joint reference
distribution is the product measure $P_Z \otimes P_U$, 
reflecting the assumption that $Z$ and $U$ are independent. The marginal distributions $P_Z$ and $P_U$ are selected to be simple and easy to sample from (e.g. Gaussian). We seek a \emph{generative map} $G: \mathcal{Z} \times \mathcal{U} \to  \mathcal{Y} \times \mathcal{X}$ that pushes forward this reference distribution to the joint data distribution $P_{Y,X}$. To enable conditional
sampling, we require $G$ to be \emph{block-triangular}.

\begin{definition}[Block-Triangular Map]
A map $G: \mathcal{Z} \times \mathcal{U} \to \mathcal{Y} \times 
\mathcal{X}$ is called \emph{block-triangular} if it admits the 
form
\begin{align}\label{eq:block-triangular}
    G(z, u) = \Bigl(G_{\mathcal{Y}}(z),\; 
    G_{\mathcal{X}}\bigl(G_{\mathcal{Y}}(z),\, u\bigr)\Bigr),
\end{align}
for maps $G_{\mathcal{Y}}: \mathcal{Z} \to \mathcal{Y}$ and 
$G_{\mathcal{X}}: \mathcal{Y} \times \mathcal{U} \to \mathcal{X}$.
\end{definition}

The block-triangular structure is precisely what allows access to
conditional distributions. Specifically, by~\cite[Thm. 2.4]{baptista2024conditional}, if a block-triangular map $G$ of the
form~\eqref{eq:block-triangular} satisfies
\begin{equation}\label{eq:push-forward-condition}
  P_{Y,X} = G \# \bigl(P_{Z} \otimes P_{U}\bigr),  
\end{equation}
where ``$\#$" denotes the pushforward operator, then, its components satisfy
\begin{equation}\label{eq:pushforward-conditional}
\begin{aligned}
G_{\mathcal{Y}}\# P_{Z} &= P_Y,\\
G_{\mathcal{X}}(y,\cdot)\# P_{U} &= P_{X \mid Y = y},\quad \textnormal{for } P_Y\textnormal{-a.e. } y \in \mathcal{Y}\,.  
\end{aligned}
\end{equation}
In other words, $G_{\mathcal{Y}}(Z)$ generates samples of $Y\sim P_Y$ from the
latent variable $Z\sim P_Z$, while $G_{\mathcal{X}}(y, U)$ generates samples from
the conditional distribution $P_{X|Y=y}$ for any fixed $y$, using  the latent variable $U \sim P_U$.
Thus, once $G_{\mathcal{X}}$ is learned, conditional sampling reduces
to a single forward pass.
Our central assumption is that such a low-dimensional block-triangular
representation of $P_{Y,X}$ exists.
\begin{assumption}[Low-Dimensional Representation]
\label{assum:low-dim-representation}
There exists a block-triangular map $G^\dagger: \mathcal{Z} \times 
\mathcal{U} \to \mathcal{Y} \times \mathcal{X}$, with components 
$G^\dagger_{\mathcal{Y}}: \mathcal{Z} \to \mathcal{Y}$ and 
$G^\dagger_{\mathcal{X}}: \mathcal{Y} \times \mathcal{U} \to \mathcal{X}$, 
such that
\begin{equation}\label{eq: rem 1}
    P_{Y,X} = G^\dagger \# \bigl(P_Z \otimes P_U\bigr)\,.
\end{equation}
\end{assumption}

\medskip 
\noindent
{\bf Problem 1:} Approximate a block-triangular generator $G$  of the form~\eqref{eq:block-triangular}, that satisfies~\eqref{eq:push-forward-condition},   from samples of $(Y, X)$.

\at{ 
\begin{remark}Assumption~\ref{assum:low-dim-representation} is imposed to ensure that the exact recovery of the conditional distribution, i.e. Problem 1,  is feasible. Our proposed approach can still be implemented when \eqref{eq: rem 1} is not true, however, the consequence is that the recovery of the conditional distribution is not exact anymore. Carrying out a full reconstruction error analysis  without Assumption~\ref{assum:low-dim-representation} is subject of future work.
\end{remark}}

\medskip
\noindent
{\bf Non-compositional generator:} 
In the formulation above, the conditional sampler $G_{\mathcal{X}}(y, U)$
takes the raw observation $y \in \mathcal{Y}$ as input directly, meaning
it operates on the full high-dimensional ambient space $\mathcal{Y}$.
While this is what enables plug-in conditional sampling, it may be
inefficient when $\mathcal{Y}$ is very high-dimensional, since the
generator must process the full observation at inference time.
An alternative formulation reduces this cost by introducing an
\emph{encoder} $\Phi_{\mathcal{Y}}: \mathcal{Y} \to \mathcal{Z}$ that
first compresses the observation $y$ into a low-dimensional latent code
$z = \Phi_{\mathcal{Y}}(y) \in \mathcal{Z}$. The conditional generator
then operates on the latent space rather than the ambient observation
space. Specifically, instead of a compositional map
$G_{\mathcal{X}}(G_{\mathcal{Y}}(z), u)$, we seek a
\emph{non-compositional} generator
$\overline{G}_{\mathcal{X}}: \mathcal{Z} \times \mathcal{U} \to
\mathcal{X}$, supported entirely on the low-dimensional latent space
$\mathcal{Z} \times \mathcal{U}$. The resulting conditional sampler is
\begin{align}\label{eq:conditional-recovery-noncomp}
    P_{X \mid Y = y} \approx
    \overline{G}_{\mathcal{X}}\!\bigl(\Phi_{\mathcal{Y}}(y),\, \cdot\,
    \bigr) \#\, P_U,
\end{align}
so that samples from $P_{X|Y=y}$ are obtained by encoding $y$ into the
latent space and then pushing $U \sim P_U$ through
$\overline{G}_{\mathcal{X}}(\Phi_{\mathcal Y}(y),\, \cdot\,)$. For this formulation to be well-posed, the encoder $\Phi_{\mathcal{Y}}$
must be a left inverse of the generator component $G_{\mathcal{Y}}^\dagger$
on the support of $P_Y$. This is the content of the following assumption. 
\begin{assumption}[Left-Invertibility of the Generator]
\label{assum:invertibility-on-support}
There exists a set $\mathcal{S}_{\mathcal{Y}} \subseteq \mathcal{Y}$
with $P_Y(\mathcal{S}_{\mathcal{Y}}) = 1$, and a measurable map
$\Phi^\dagger_{\mathcal{Y}}: \mathcal{S}_{\mathcal{Y}} \to \mathcal{Z}$
such that
\begin{align*}
    \Phi^\dagger_{\mathcal{Y}} \circ G^\dagger_{\mathcal{Y}}(z) &= z,
    \qquad \text{for all } z \in
    G^{\dagger\,-1}_{\mathcal{Y}}(\mathcal{S}_{\mathcal{Y}}), \\
    G^\dagger_{\mathcal{Y}} \circ \Phi^\dagger_{\mathcal{Y}}(y) &= y,
    \qquad \text{for all } y \in \mathcal{S}_{\mathcal{Y}},
\end{align*}
where $G^{\dagger\,-1}_{\mathcal{Y}}(\mathcal{S}_{\mathcal{Y}})$ denotes
the preimage of $\mathcal{S}_{\mathcal{Y}}$ under $G^\dagger_{\mathcal{Y}}$.
\end{assumption}
Under this assumption, the non-compositional generator
$G:\mathcal Z \times \mathcal U \to \mathcal Y \times \mathcal X$ can be defined by
\begin{equation}\label{eq:non-compose-G}
    G(z, u) = \Bigl(G_{\mathcal{Y}}(z),\; 
    \overline G_{\mathcal{X}}(z,u)\Bigr). 
\end{equation}

\medskip 
\noindent
{\bf Problem 2:} Approximate a block-triangular generator $G$ of the form~\eqref{eq:non-compose-G}, that satisfies~\eqref{eq:push-forward-condition}, and an encoder $\Phi_{\mathcal Y}$ that is the inverse of $G_{\mathcal Y}$, in the sense defined in Assumption~\ref{assum:invertibility-on-support},   from samples of $(Y, X)$.

\at{
\begin{remark}
 Assumption~\ref{assum:invertibility-on-support} is required to ensure exact recovery of the conditional distribution through a non-compositional generator, i.e. existence of a solution to Problem 2. The failure of the assumption leads to a degradation error that is quantified in Proposition \ref{prop: conditional error bounds} via the conditional representation error~\eqref{eq: conditional representation error}.
\end{remark}
}

\subsection{Background on Wasserstein AutoEncoders}\label{sec: background WAE}
In this section, we give a brief review on WAEs~\cite{tolstikhin2018wasserstein}. To this end, we consider the unconditional setup where, without incorporating the observation variable $Y$, the goal is to train a deterministic decoder (generator) $G:\mathcal{U}\to\mathcal{X}$ such that
\[
G\#P_U= P_X.
\]
That is to say, $G(U)$ generates samples of $X \sim P_X$ through samples of $U \sim P_U$. 
In the WAE framework, the desired generator is obtained by minimizing the Wasserstein distance 
\[
\inf_G\,W_c(P_X, G\#P_U),
\]
over all generator maps $G$, 
for a prescribed cost function $c:\mathcal{X}\times \mathcal{X}\to \mathbb{R}_{\geq 0}$.
The key idea in WAE is to form an equivalent expression for the Wasserstein distance that is amenable to optimization, via the introduction of an encoder $Q_{U|X}$ that couples $U$ and $X$. Indeed, as shown in ~\cite[Thm. 2.4]{tolstikhin2018wasserstein},
\begin{align*}
&W_c(P_X, G\#P_U)
=
\displaystyle\inf_{Q: Q_U = P_U}
\mathbb{E}
~c(X, G(U)),\\
& \textit{where} \quad X\sim P_X,\quad U\sim Q_{U\mid X},
\end{align*}
and $Q_U=\int_{\mathcal{X}}Q_{U|X=x}P_X(dx)$,
denotes the marginal of $U$ (also called aggregated posterior) induced by the encoder.
Since enforcing the constraint $Q_U = P_U$ exactly is intractable in practice, WAEs relax it by introducing a penalty term, leading to the regularized minimization problem
\begin{align}
&\displaystyle\inf_{G}\inf_{Q}
\mathbb{E}~c(X, G(U))+
\lambda \, D_\mathcal{U}(Q_U, P_U),\label{eq:wae_loss}
\\
& \textit{where} \quad X\sim P_X,\quad U\sim Q_{U\mid X},\nonumber
\end{align}
where $D_\mathcal{U}$ is a chosen discrepancy between distributions supported on the latent space and $\lambda > 0$ is a regularization parameter.
The first term promotes accurate reconstruction of the data through the decoder, while the second term encourages the aggregated posterior $Q_U$ to match the prior $P_U$. 
The choice of $D_\mathcal{U}$ is flexible and may correspond to any divergence or integral probability metric (e.g., adversarial objectives or kernel-based discrepancies). Importantly, this discrepancy is evaluated on distributions over the low-dimensional latent space $\mathcal U$, making it computationally tractable and protected from the curse of dimensionality.

In the next section, 
we extend the WAE framework reviewed above
to the conditional simulation setup introduced in Sec.~\ref{sec: problem setup}, particularly to develop algorithms solving {Problem 1} and {Problem 2}.

\section{Conditioning via Wasserstein AutoEncoders}\label{sec:CWAE}

\subsection{Joint vs. Conditional OT Costs}\label{sec: BTT}

In this section, we introduce two optimization objectives for learning 
the block-triangular maps that solve {Problem~1} and {Problem~2}. The first is based on the Wasserstein distance 
between joint distributions, while the second is based on a strictly 
stronger notion of conditional Wasserstein distance.

\medskip
\noindent
{\bf Joint OT cost:} One natural approach to solving {Problem 1} is to minimize the
\emph{joint OT cost}
\begin{align}\label{eq: Joint cost}W_c(P_{Y,X}, G\#P_{Z,U}),
\end{align}
\at{where $P_{Z,U} := P_Z\otimes P_U$ and $c: (\mathcal{Y}\times \mathcal{X}) \times (\mathcal{Y}\times \mathcal{X}) \to \mathbb{R}_{\geq 0}$ is a prescribed joint cost}, over all maps $G$ that admit the block-triangular structure in~\eqref{eq:block-triangular}.   At optimality, that is, when the objective reaches zero \at{(thanks to Assumption~\ref{assum:low-dim-representation})}, the
pushforward condition~\eqref{eq:push-forward-condition} holds exactly, and the
conditional distributions are recovered through $P_{X|Y=y}=G_{\mathcal{X}}(y, \cdot)_{\#} P_U$ as in~\eqref{eq:pushforward-conditional}.
However, when the minimum value is nonzero, the joint cost provides
no direct control over the quality of the individual conditional
approximations $G_{\mathcal{X}}(y, \cdot)_{\#} P_U$ relative to the
target conditionals $P_{X|Y=y}$, for any particular value of  $y$.  This
limitation motivates the conditional OT cost introduced
below.

\medskip
\noindent
{\bf Conditional OT cost:}
Rather than measuring discrepancy at the level of the joint distribution,
we can instead integrate the pointwise error between
$G_{\mathcal{X}}(y, \cdot)_{\#} P_U$ and $P_{X|Y=y}$ over all
realizations of $y$:
\begin{align}\label{eq: COT cost} \int_{\mathcal{Y}} W_{c_{\mathcal{X}}}\left(P_{X\mid Y=y}, G_{\mathcal{X}}(y,\cdot)\#P_U\right)P_Y(dy),
\end{align}
where $c_{\mathcal{X}}:\mathcal{X}\times\mathcal{X}\to\mathbb{R}_{\geq 0}$ is a marginal cost.
This quantity, related to the notion of
conditional Wasserstein distance~\cite{carlier2016vector,
hosseini2025conditional}, provides a more direct  control over the quality of the conditional approximation than the joint cost.
To make this precise, suppose the joint cost $c$ is additive and separable,
\begin{equation}\label{eq: cost in YX}
c((y,x),(\at{y', x'})):= c_\mathcal{Y}(y,\at{y'})+c_\mathcal{X}(x,\at{x'}).
\end{equation}
Then, \at{with $P_{Y,U}:=P_Y \otimes P_U$,} the following inequality holds:
\begin{align*}
&W_c(P_{Y,X},G\#P_{Y,U})\\
&\leq \int_{\mathcal{Y}} W_{c_{\mathcal{X}}}\left(P_{X\mid Y=y}, G_{\mathcal{X}}(y,\cdot)\#P_U\right)P_Y(dy),
\end{align*}
implying that the conditional cost~\eqref{eq: COT cost} is a stronger notion of error than the joint cost. The inequality follows using the definition of the joint Wasserstein distance and picking a non-optimal coupling that induces the right-hand-side.

However, the objective~\eqref{eq: COT cost} is not amenable to our latent variable and generator framework, unless we set $Z=Y$, or equivalently restrict  $G_\mathcal{Y}=\mathrm{Id}_{\mathcal{Y}}$. To address this issue and retain the flexibility of our framework, we consider the setup of Problem 2 where $Y$ is encoded into the latent variable $Z$ through $\Phi_\mathcal{Y}$. This encoder induces a coupling between $Y$ and
$Z$, and hence a well-defined conditional distribution $P_{X|Z=z}$. This motivates the following surrogate objective, which we call the
\emph{latent conditional OT cost}:
\begin{align}
\mathcal{R}_Z(\overline{G}_\mathcal{X};&\Phi_\mathcal{Y})\label{eq: compressed mean conditional cost}\\
&:=\int_{\mathcal{Z}} W_{c_\mathcal{X}}\left(P_{X\mid Z=z}, \overline{G}_{\mathcal{X}}(z, \cdot) \# P_U\right)P_Z(dz),\nonumber
\end{align}
where $\Phi_{\mathcal{Y}}$ enters the objective implicitly through its
effect on the conditional distribution $P_{X|Z}$. 
 The term ``latent''
reflects the fact that we measure conditional discrepancies with respect
to $Z = \Phi_{\mathcal{Y}}(Y)$, the low-dimensional encoding of $Y$,
rather than $Y$ itself. This cost serves as a tractable surrogate for
the conditional cost~\eqref{eq: COT cost}, recovering it exactly when $Z = Y$.

Moreover, whenever the encoder is exact, i.e., $\Phi_\mathcal{Y}$ is a valid solution of Problem 2, 
the latent conditional OT cost \eqref{eq: compressed mean conditional cost} matches the \emph{conditional OT cost}
\begin{align}
\mathcal{R}&_Y(\overline{G}_\mathcal{X};\Phi_\mathcal{Y})\label{eq: mean conditional cost}\\
&:=\int_{\mathcal{Y}} W_{c_\mathcal{X}}\left(P_{X\mid Y=y}, \overline{G}_{\mathcal{X}}(\Phi_\mathcal{Y}(y), \cdot) \# P_U\right)P_Y(dy),\nonumber
\end{align}
This is shown in the following proposition, where we also quantify the gap between the two objectives by the following notion of 
\emph{conditional representation error}
\begin{align}\label{eq: conditional representation error}
\mathcal{E}(\Phi_\mathcal{Y}):=\int_\mathcal{Y} W_{c_\mathcal{X}}\left(P_{X\mid Y=y}, P_{X\mid Z = \Phi_{\mathcal{Y}}(y)}\right)P_Y(dy),
\end{align}
that measures  the average mismatch between the target conditional distributions and their latent analogue. 
\begin{proposition}\label{prop: conditional error bounds}
For any encoder approximation $\Phi_\mathcal{Y}:\mathcal{Y}\to\mathcal{Z}$ encoding $Z=\Phi_\mathcal{Y}(Y)$ and transport map $\overline{G}_\mathcal{X}:\mathcal{Z}\times\mathcal{U}\to\mathcal{X}$, it holds that:
\begin{align}\label{eq: upper and lower bound}
\begin{split}
\mathcal{R}_Z(\overline{G}_\mathcal{X};\Phi_\mathcal{Y})&\leq \mathcal{R}_Y(\overline{G}_\mathcal{X};\Phi_\mathcal{Y})\\
&\leq \mathcal{R}_Z(\overline{G}_\mathcal{X};\Phi_\mathcal{Y})+\mathcal{E}(\Phi_\mathcal{Y}).
\end{split}
\end{align}
Moreover, if $\Phi_\mathcal{Y}$ is a valid solution of Problem 2, it follows that $\mathcal{E}(\Phi_\mathcal{Y}) = 0$ and $\mathcal{R}_Y(\overline{G}_\mathcal{X};\Phi_\mathcal{Y}) = \mathcal{R}_Z(\overline{G}_\mathcal{X};\Phi_\mathcal{Y})$.
\end{proposition}
\begin{proof}
Introduce the compact notation 
\[\widetilde{P}_z:=\overline{G}_\mathcal{X}(z,\cdot)\# P_U.\]
We first show the lower bound in~\eqref{eq: upper and lower bound}. Recall that, 
\[
P_{X\mid Z = z} = \int_{\mathcal{Y}}P_{X\mid Y = y}P_{Y\mid Z = z}(dy).
\]
Then, by the convexity  of the optimal transport cost $W_{c_\mathcal{X}}(\cdot, \widetilde{P}_z)$~\cite[Thm. 4.8]{villani2009optimal},
\[
W_{c_\mathcal{X}}(P_{X\mid Z=z}, \widetilde{P}_z)\leq \int_{\mathcal{Y}}W_{c_\mathcal{X}}(P_{X\mid Y=y}, \widetilde{P}_z)P_{Y\mid Z =z}(dy).
\]
Integrating over $P_{Z}(dz)$, and noticing that in the integrated right hand side $z=\Phi_\mathcal{Y}(y)$ along the support of $P_{Y,Z}(dy,dz)$, ultimately yields the lower bound.

To obtain the upper bound, by triangle inequality,
\begin{align*}
\begin{split}
W_{c_{\mathcal{X}}}(P_{X\mid Y=y}, \widetilde{P}_{\Phi_{\mathcal{Y}}(y)})&\leq W_{c_{\mathcal{X}}}(P_{X\mid Y=y}, P_{X\mid Z=\Phi_{\mathcal{Y}}(y)})\\
&+ W_{c_{\mathcal{X}}}(P_{X\mid Z=\Phi_{\mathcal{Y}}(y)}, \widetilde{P}_{\Phi_{\mathcal{Y}}(y)}).
\end{split}
\end{align*}
Then, integrating with respect to $P_{Y}(dy)$ and applying the pushforward identity $P_{Z} = \Phi_{\mathcal{Y}} \# P_Y$ to the second term of the integrated right hand side yields the upper bound.

Ultimately, suppose that $\Phi_{\mathcal{Y}}$ solves Problem 2. Then, for $P_Y$-a.e. $y\in\mathcal{Y}$, by Assumptions~\ref{assum:low-dim-representation} and~\ref{assum:invertibility-on-support},
\begin{align*}
\begin{split}
P_{X\mid Z = \Phi_\mathcal{Y}(y)}&= G^{\dag}_{\mathcal{X}}(G^\dag_\mathcal{Y}(\Phi_\mathcal{Y}(y)), \cdot)_\# P_U\\
& =G^\dag_\mathcal{X}(y, \cdot) \# P_U = P_{X\mid Y = y}.
\end{split}
\end{align*}
Equivalently, $X\perp Y \mid Z$ and $\mathcal{E}(\Phi_\mathcal{Y}) = 0$.
\end{proof}
Now that we have introduced the objective functions, the next two sections provide the details behind their minimization. 
\subsection{The WAE for Conditioning}\label{sec: conditioning with WAE}
The Wasserstein Autoencoder for Conditioning (WAE-C) trains a deterministic decoder $G_{\mathcal{X}}^*:\mathcal{Y}\times\mathcal{U}\to\mathcal{X}$ to generate samples from $P_{X^*\mid Y=y}\approx P_{X\mid Y=y}$ where
\begin{align}\label{eq: conditional reconstructions 1}
P_{X^*\mid Y=y} := G_{\mathcal{X}}^*(y,\cdot) \#P_{U}
\end{align}
for any prescribed conditioning value $y\in \mathcal{Y}$.
The training objective is formulated as the 
minimization of the joint OT cost in~\eqref{eq: Joint cost}
over block-triangular decoders $G$ parameterized as in~\eqref{eq:block-triangular}.
Following the WAE framework reviewed in Sec.~\ref{sec: background WAE}, for a chosen discrepancy $D_{\mathcal{Z}\times \mathcal{U}}$ between distributions supported on the joint latent space $\mathcal{Z}\times\mathcal{U}$, minimizing the joint OT cost is relaxed and formulated as the following regularized optimization problem
\begin{align}
&\displaystyle\inf_{G}\inf_{Q}\mathbb{E}\big[ c((Y,X), G(Z,U)) \big]+
\lambda \, D_{\mathcal{Z}\times \mathcal{U}}(Q_{Z,U}, P_{Z,U}),\nonumber\\
&\textit{where}\quad Y,X\sim P_{Y,X},\quad Z,U \sim Q_{Z,U|Y,X},\label{eq: WAE-C objective}
\end{align}
and 
\begin{align}
&G(z, u) = \left(G_\mathcal{Y}(z), G_\mathcal{X}\left(G_\mathcal{Y}(z), u\right)\right),\label{eq: additional constraints}\\
&Q_{Z,U} := \displaystyle\int_{\mathcal{Y}\times \mathcal{X}}Q_{Z,U\mid Y=y,X =x}P_{Y,X}(dy,dx),\nonumber\\
&P_{Z,U} = P_{Z}\otimes P_{U}.\nonumber
\end{align}


Comparing~\eqref{eq:wae_loss} and~\eqref{eq: WAE-C objective}, WAE-C can be interpreted as an instance of WAE applied to the joint space, with the additional constraints of a block-triangular decoder and an independently coupled latent variable pair as in~\eqref{eq: additional constraints}. These structural constraints are essential to recover the desired conditional reconstructions in~\eqref{eq: conditional reconstructions 1}.

\subsection{Conditional Wasserstein Autoencoders}\label{sec: sequential CWAE}

The Conditional Wasserstein Autoencoder (CWAE) seeks to learn approximations $\Phi_\mathcal{Y}^*:\mathcal{Y}\to\mathcal{Z}$ and $\overline{G}_\mathcal{X}^*:\mathcal{Z}\times\mathcal{U}\to\mathcal{X}$ to generate samples from $P_{X^*\mid Y^* = y}\approx P_{X\mid Y = y}$ where 
\begin{align}\label{eq: conditional reconstructions 2}
P_{X^*\mid Y^* = y}:= \overline{G}_\mathcal{X}^*(\Phi_\mathcal{Y}^*(y), \cdot) \# P_U,
\end{align}
for any prescribed conditioning value $y\in\mathcal{Y}$. We divide such a procedure in two steps which can then be combined together to yield a collective minimization objective.

\medskip
\emph{Learning $\Phi_\mathcal{Y}^*$}: This is achieved by formulating the WAE objective over the marginal distribution $P_Y$: 
\begin{align}\label{eq: Y cost}
&W_{c_{\mathcal{Y}}}\left(P_Y, G_{\mathcal{Y}}\#P_Z\right)= \inf_{\substack{\Phi_\mathcal{Y}}}\mathbb{E}~c_\mathcal{Y}(Y, G_\mathcal{Y}(\Phi_\mathcal{Y}(Y))),\\
&\textit{such that}\quad \Phi_{\mathcal{Y}}\#P_Y=P_Z,\nonumber
\end{align}
in terms of the cost $c_{\mathcal{Y}}:\mathcal{Y}\times\mathcal{Y}\to\mathbb{R}_{\geq 0}$. Replacing the constraint by the penalty term yields the minimization problem
\begin{align}\label{eq: CWAE1 objective}
\inf_{G_\mathcal{Y}}\inf_{\Phi_\mathcal{Y}}\mathbb{E}~c_{\mathcal{Y}}(Y,G_\mathcal{Y}(\Phi_\mathcal{Y}(Y)))+\lambda D_{\mathcal{Z}}\left(P_Z,\Phi_\mathcal{Y} \# P_Y\right).
\end{align}
 Under Assumptions~\ref{assum:low-dim-representation} and~\ref{assum:invertibility-on-support}, an optimal minimizer $G_\mathcal{Y}^*$ of \eqref{eq: Y cost} and an associated optimal encoder $\Phi_\mathcal{Y}^*$ satisfy  $G_\mathcal{Y}^*\circ \Phi_\mathcal{Y}^*(y)=y$ for $P_{Y}$-a.e. $y\in\mathcal{Y}$.

\medskip
\emph{Learning $\overline{G}_\mathcal{X}^*$}: After learning  $\Phi_\mathcal{Y}^*$, let $Z = \Phi_\mathcal{Y}^*(Y)$. This relationship induces a coupling between $X$ and $Z$, which is denoted by $P_{Z,X}$. With this coupling, we can define and use the latent conditional OT cost~\eqref{eq: compressed mean conditional cost} to learn $\overline{G}_\mathcal{X}^*$. Towards achieving this, we provide an alternative representation of~\eqref{eq: compressed mean conditional cost} via the introduction of an encoder $Q_{U\mid Z,X}$, stated in the following Theorem. 
\begin{theorem}\label{thm: CWAE result} For any $\Phi_\mathcal{Y}^*$ inducing the coupling $Z,X\sim P_{Z,X}$:
\begin{align}
&\mathcal{R}_Z(\overline{G}_\mathcal{X};\Phi_\mathcal{Y}^*)= \inf_{\substack{Q:Q_{Z,U} = P_{Z,U}\\}}\mathbb{E}~ c_{\mathcal{X}}(X,\overline{G}_\mathcal{X}(Z,U)),\nonumber\\
&\textit{where}\quad  Z,X\sim P_{Z,X},\quad U\sim Q_{U\mid Z,X},\label{eq: COT cost rep}
\end{align}
over the block-triangular encoders and \at{aggregated} posterior
\begin{align}
&Q_{Z',U\mid Z,X} :=Q_{U\mid Z,X}\delta(Z'-Z),\label{eq: constraint encoders}\\
&Q_{Z,U} := \int_{\mathcal{Z}\times\mathcal{X}}Q_{Z',U\mid Z=z,X=x}P_{Z,X}(dz,dx).\nonumber
\end{align}
\end{theorem}

\begin{proof}
Introduce the transport map $\widetilde{G}:\mathcal{Z}\times\mathcal{U}\to \mathcal{Z}\times \at{\mathcal{X}}$ defined as
\[
\widetilde{G}(z,u) = \left(z,\overline{G}_\mathcal{X}(z,u)\right),
\]
for all $(z,u)\in \mathcal{Z}\times\mathcal{U}$. Then, we can rewrite the conditional OT cost as a joint OT cost
\[
\mathcal{R}_Z(\overline{G}_\mathcal{X};\Phi_\mathcal{Y}^*) = W_{\widetilde{c}}\left(P_{Z,X}, \widetilde{G}\# P_{Z,U}\right),
\]
 in terms of the  cost 
\[
\widetilde{c}((z,x),(\at{z'},x')):=\begin{cases}
c_{\mathcal{X}}(x,x'),\quad  \textnormal{if }z=z',\\
+\infty,\quad\textnormal{if }z\neq z', 
\end{cases}
\]
for all $((z,x),(z',x'))\in \left(\mathcal{Z}\times\mathcal{X}\right)\times \left(\mathcal{Z}\times\mathcal{X}\right)$. Then, by Theorem 1 of \cite{tolstikhin2018wasserstein}, the joint objective can be rewritten in terms of a joint encoder $Q_{Z',U\mid Z,X}$ as
\begin{align*}
&\mathcal{R}_Z(\overline{G}_\mathcal{X};\Phi_\mathcal{Y}^*)=\inf_{Q:Q
_{Z,U}=P_{Z,U}}\mathbb{E}~\widetilde{c}((Z,X),\widetilde{G}(Z',U)),\\
&\textit{where}\quad Z,X\sim P_{Z,X},\quad Z',U\sim Q_{Z',U\mid Z,X}.
\end{align*}
Finally, noticing that $\widetilde{c}$ is infinity whenever $Z'\neq Z$, the formulation of the joint cost can be written exclusively in terms of $c_\mathcal{X}(X,\overline{G}_{\mathcal{X}}(Z,U))$ under the constraint over the class of admissible encoders as in \eqref{eq: constraint encoders}, completing the proof.
\end{proof}

In light of Theorem~\ref{thm: CWAE result}, $\overline{G}_\mathcal{X}^*$ is then obtained by solving the regularized minimization problem
\begin{align}
&\displaystyle\inf_{\overline{G}_\mathcal{X}}\inf_{Q}\mathbb{E}c_{\mathcal{X}}(X,\overline{G}_\mathcal{X}(Z,U))+ \lambda D_{\mathcal{Z}\times\mathcal{U}}\left(Q_{Z,U}, P_{Z,U}\right),\nonumber\\
&\textit{where}\quad Z,X\sim P_{Z,X},\quad U\sim Q_{U\mid Z,X},\label{eq: CWAE objective 2}
\end{align}
where we recall that we encode $Z=\Phi_\mathcal{Y}^*(Y)$ and that we set $P_{Z,U}=P_Z\otimes P_U$.
\begin{remark}\label{rem: change of parametrization}
The minimization problem in~\eqref{eq: CWAE objective 2} can also be implemented with the compositional parametrization
\begin{align}\label{eq: compositional reparam}
\overline{G}_\mathcal{X}=G_{\mathcal{X}}\circ\left(G_\mathcal{Y}^*\times\mathrm{Id}_\mathcal{U}\right).
\end{align}
using the approximation $G_\mathcal{Y}^*$ of $G_\mathcal{Y}^\dag$ in~\eqref{eq: CWAE1 objective}. This procedure then performs the conditional generation task associated to Problem 1, bypassing the preliminary use of the encoder approximation $\Phi_\mathcal{Y}^*$. However, as explained in Sec.~\ref{sec: problem setup}, we stress that doing so loses the main practical advantage of solving Problem 2, i.e., learning a conditional transport map 
supported on the fully low-dimensional latent space $\mathcal{Z}\times\mathcal{U}$.
\end{remark}

\section{Numerical experiments}\label{sec:numerics}
\subsection{Numerical procedure}\label{sec:numerical_alg}
In the CWAE implementations, the objectives associated to the minimization problems of~\eqref{eq: Y cost} and~\eqref{eq: COT cost rep} introduced in Sec.~\ref{sec: sequential CWAE} are summed and minimized jointly over classes of encoders $\Phi_\mathcal{Y},Q$ and decoders $G_\mathcal{Y},\overline{G}_\mathcal{X}$. The constraints $\Phi_\mathcal{Y}\#P_Y = P
_Z$ and $Q_{Z,U} = P_{Z,U}$ are enforced simultaneously solving the regularized minimization problem, for a pre-fixed penalty parameter $\lambda>0$,
\begin{align}
&\displaystyle\inf_{G_\mathcal{Y}, \overline{G}_\mathcal{X}}\inf_{\Phi_\mathcal{Y},Q}\mathbb{E}~
c((Y,X),(\at{G_\mathcal{Y}(}\Phi_\mathcal{Y}(Y)),\overline{G}_\mathcal{X}(\Phi_\mathcal{Y}(Y),U)))\nonumber\\
&\qquad\qquad+\lambda D_{\mathcal{Z}\times\mathcal{U}}\left(Q_{Z,U}, P_{Z,U}\right),\nonumber\\
&\textit{where}\quad Y,X\sim P_{Y,X},\quad U\sim Q_{U\mid \Phi_\mathcal{Y}(Y),X}.\label{eq: CCWAE reg objective}
\end{align}
over block-triangular encoders and aggregated posterior
\begin{align*}
&Q_{Z,U\mid Y,X} :=Q_{U\mid Z,X}\delta(Z-\Phi_\mathcal{Y}(Y)),\\
& Q_{Z,U} := \int_{\mathcal{Y}\times\mathcal{X}}Q_{Z,U\mid Y=y,X=x}P_{Y,X}(dy,dx).
\end{align*}
In this setup as well,
~\eqref{eq: CCWAE reg objective} can be implemented with a compositional parametrization akin to~\eqref{eq: compositional reparam} and Remark~\ref{rem: change of parametrization} applies analogously.
In the numerical experiments, we consider fully deterministic CWAE architectures employing several encoder and decoder variants. Specifically, given samples $(y^i, x^i) \sim P_{Y,X}$ and $(z^i, u^i) \overset{\mathrm{i.i.d.}}{\sim} \mathcal{N}(0, \mathrm{Id}_{d_\mathcal{Z} + d_\mathcal{U}})$, we minimize the empirical approximation of~\eqref{eq: CCWAE reg objective} under a quadratic cost:
\begin{equation}\label{eq:ultimate-loss}
\begin{aligned}
    \frac{1}{N}\sum_{i=1}^N & \left (\|x^i-\hat{x}^i\|^2+\|y^i-\hat{y}^i\|^2\right) \\
    +\lambda D_{\mathcal{Z}\times\mathcal{U}} &\left(\frac{1}{N}\sum_{i=1}^N\delta_{(\hat{z}^i,\hat{u}^i)},\frac{1}{N}\sum_{i=1}^N\delta_{(z^i, u^i)} \right)
\end{aligned}
\end{equation}
where $\hat{z}^i=\Phi_{\mathcal{Y}}(y^i)$, $\hat{y}^i = G_\mathcal{Y}(\hat{z}^i)$, and the three architectural variants are defined as follows:
\begin{equation*}
\begin{alignedat}{3}
& \textbf{CWAE1:} \quad & \hat{x}^i &= \overline{G}_\mathcal{X}(\hat{z}^i,\hat{u}^i), \quad & \hat{u}^i &= \overline{\Phi}_\mathcal{X}(y^i,x^i), \\
& \textbf{CWAE2:} \quad & \hat{x}^i &= G_\mathcal{X}(\hat{y}^i,\hat{u}^i), \quad & \hat{u}^i &= \Phi_\mathcal{X}(\hat{z}^i,x^i), \\
& \textbf{CWAE3:} \quad & \hat{x}^i &= \overline{G}_\mathcal{X}(\hat{z}^i,\hat{u}^i), \quad & \hat{u}^i &= \Phi_\mathcal{X}(\hat{z}^i,x^i).
\end{alignedat}
\end{equation*}


Samples from the approximate conditional distribution of $P_{X\mid Y=y}$ are subsequently generated by drawing $u \sim \mathcal{N}(0, \mathrm{Id}_{d_\mathcal{U}})$ and evaluating

\begin{equation*}
    \begin{cases}
    \overline{G}_\mathcal{X}^*(\Phi_\mathcal{Y}^*(y) ,u) & \text{for \textbf{CWAE1} and \textbf{CWAE3}}, \\
    G_\mathcal{X}^*(y ,u) & \text{for \textbf{CWAE2}}.
    \end{cases}
\end{equation*}

\begin{figure}[t]
    \centering
    \begin{overpic}[width=1\linewidth, trim={0 15 15 15}, clip]{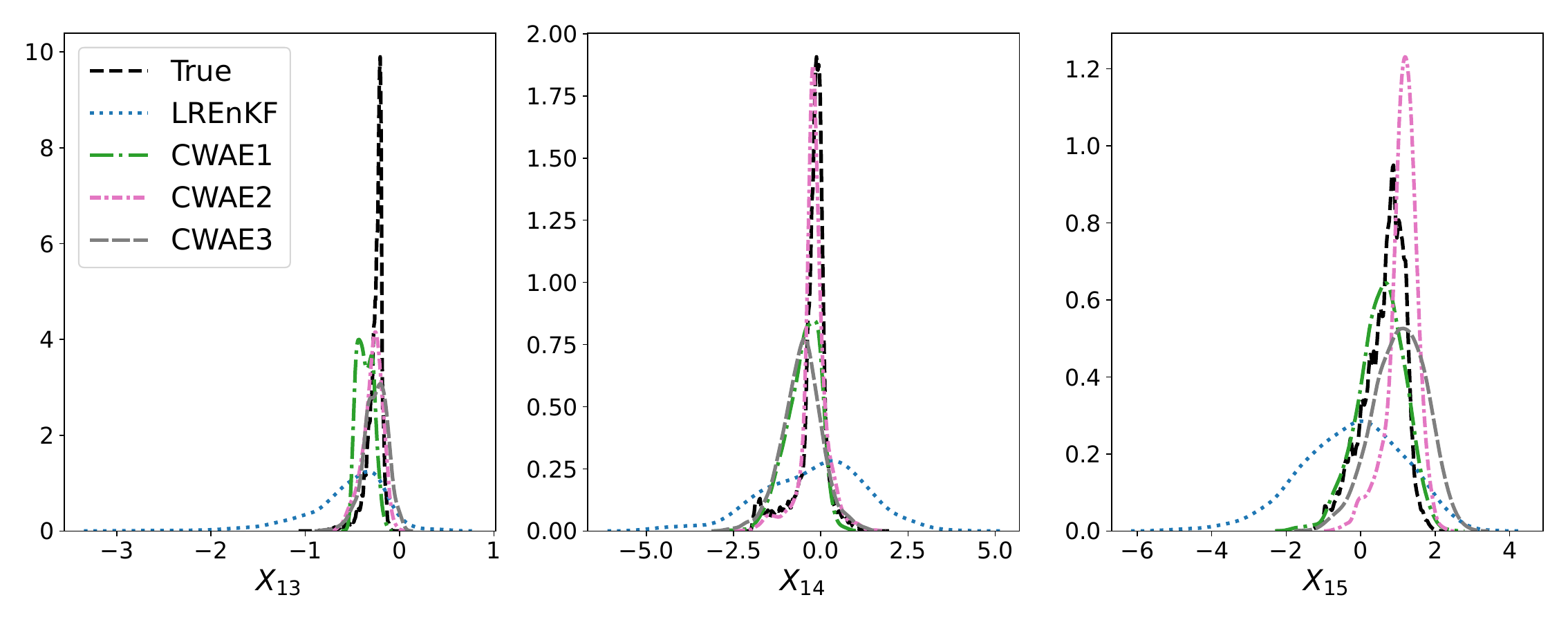}
    \put(0,20){\makebox(0,0)[c]{\rotatebox{90}{\scriptsize Density}}}
  \end{overpic}

    \caption{ 
    Numerical results for the synthetic example in Sec.~\ref{sec: Low-Dimensional Latent Structure Embedded in a High-Dimensional State}. The figure shows the sample
distributions for the last three states left to right for $y = \mathbf{1}_{d_\mathcal{Y}}$.
    }
    \label{fig:h2l_fig_particle}
\end{figure}

\subsection{Experiment Setup}
We compared the three CWAE variants against LREnKF~\cite[Algorithm 1]{leprovost2022lowrank} an algorithm which employs a low-rank approximation of the EnKF via an LIS approach. In all experiments, a GAN-based loss~\cite{GooPouBen14} was employed for the penalty term in~\eqref{eq: CCWAE reg objective}, utilizing the Jensen-Shannon divergence. 
\at{The network architecture, optimizer, learning rate, and the remaining
training hyperparameters were selected by automated hyperparameter
tuning, carried out using the SMAC3~\cite{smac3} and
Optuna~\cite{optuna_2019}.}
The numerical code used to produce the results is available online.\footnote{\url{https://github.com/Mohd9485/CWAE}}

\subsection{Nonlinear Manifold Embedding With Cubic Observations}

\label{sec: Low-Dimensional Latent Structure Embedded in a High-Dimensional State}

In this experiment, we construct a synthetic example designed to assess the
performance of the proposed method in settings where the
full ambient state $X$
is constructed using a low-dimensional manifold.
Let $\psi : \mathbb{R}^{d_\mathcal{U}} \rightarrow
\mathbb{R}^{d_{\mathcal{X}}-d_{\mathcal{U}}}$ be a 
embedding map which lifts the latent variables into the remaining coordinates
of the ambient space. The full ambient state $X \in \mathbb{R}^{d_\mathcal{X}}$ is then
constructed as
\[
X = \begin{bmatrix} X_1 \\ \vdots \\ X_{d_\mathcal{U}} \\
\psi(X_1,\ldots,X_{d_\mathcal{U}}) \end{bmatrix} + \gamma V, \quad
\begin{bmatrix} X_1 \\ \vdots \\ X_{d_\mathcal{U}}
\end{bmatrix}\sim  \mathcal{N}(0,
\mathrm{Id}_{d_\mathcal{U}})
\]
where $V$ is a $d_{\mathcal{X}}$-dimensional standard Gaussian random
variable. Thus, although $X$ resides in $\mathbb{R}^{d_\mathcal{X}}$, it is
constrained to lie on a $d_\mathcal{U}$-dimensional manifold. The observation vector $Y \in \mathbb{R}^{d_\mathcal{Y}}$
is given by
\[
Y = h\!\left( \begin{bmatrix} X_1 & \dots & X_{d_\mathcal{Y}}
\end{bmatrix}^\top \right) + \sigma \varepsilon,
\]
where $h(x) = x^3$ is an element-wise cubic function applied to the first
$d_\mathcal{Y}$ components of $X$, $\varepsilon \sim \mathcal{N}(0,
\mathrm{Id}_{d_\mathcal{Y}})$ is a standard Gaussian observation noise, and the
dimensions satisfy $d_\mathcal{X} \geq d_\mathcal{Y} \geq d_\mathcal{U}$.
Consequently, the observations directly reveal the underlying low-dimensional
latent structure.

We set $d_\mathcal{X} = 5 \times d_\mathcal{U}$, $d_\mathcal{Y} = 2 \times
d_\mathcal{U}$, $\gamma = 10^{-2}$, and $\sigma = 4 \times 10^{-1}$. The
results are depicted in Figure~\ref{fig:h2l_fig_particle} for $d_\mathcal{U}
= 3$, $N = 1000$, and $y = \mathbf{1}_{d_\mathcal{Y}}$, where the sample
distributions for the last three states are presented. The ground-truth
posterior distribution is computed by simulating Sequential Important Resampling (SIR)~\cite{doucet09} with $5 \times 10^6$
samples. The numerical results indicate that the CWAE variants outperform
LREnKF across the considered experimental settings.

To further investigate the robustness of the CWAE methods, we implement the
experiment across varying problem dimensions and report the results in
Table~\ref{tab:distance_comparison}, averaged over $10$ independent
simulations. The results confirm that all CWAE variants consistently
outperform LREnKF; notably, CWAE2 exhibits stable performance relative to the
other two variants. A more rigorous analysis of this behavior is left as a
direction for future work.

\begin{table}[h]
\centering
\caption{$W_2$ error across methods and problem dimensions for the nonlinear embedded example in Sec.~\ref{sec: Low-Dimensional Latent Structure Embedded in a High-Dimensional State}.}
\begin{tabular}{ccccc}
\hline
\textbf{$d_{\mathcal{X}}$} & \textbf{LREnKF} & \textbf{CWAE1} & \textbf{CWAE2} & \textbf{CWAE3} \\
\hline
10  & 3.3855 & 1.1911 & \textbf{1.1234} & 1.9790 \\
20 & 3.7908 & 1.7968 & \textbf{1.5552} & 2.4946 \\
30 & 4.6513 & 4.0965 & \textbf{3.5174} & 4.0098 \\
40 & 5.9057 & 4.7601 & \textbf{4.4719} & 5.1238 \\
50 & 7.4993 & \textbf{5.6551} & 5.6583 & 6.5567 \\
\hline
\end{tabular}
\label{tab:distance_comparison}
\end{table}

\subsection{Spherical Posterior Example}
\label{sec:cone_example}

\begin{figure}[t]
    \centering
    \begin{overpic}[width=1\columnwidth,trim={70 10 70 70},clip]{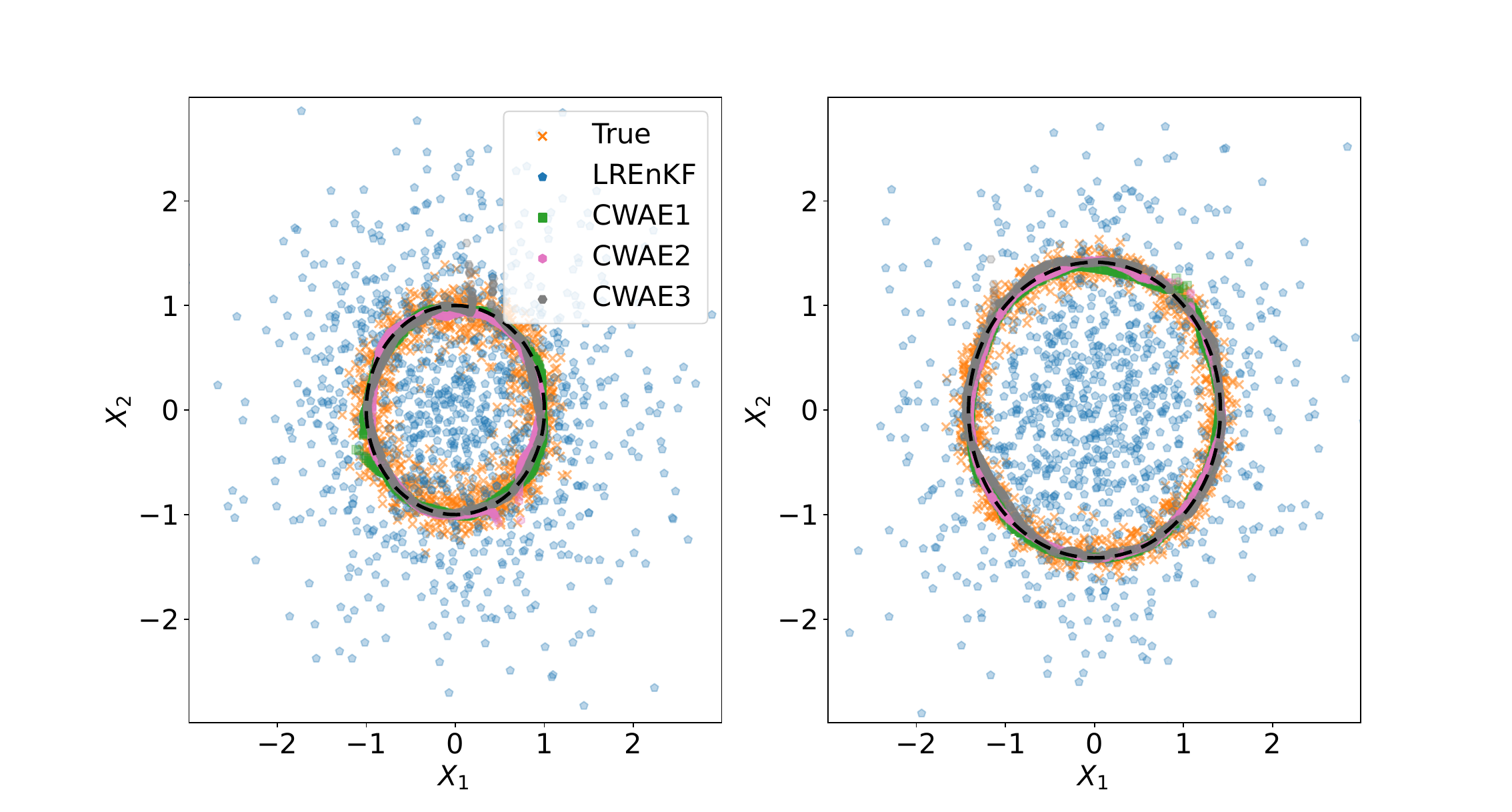}

    \put(14,48){\makebox(0,0)[c]{$y=1$}}

    \put(64,48){\makebox(0,0)[c]{$y=2$}}
  \end{overpic}
    \caption{Numerical results for the spherical posterior example in Sec.~\ref{sec:cone_example}. Both panels show the sample distributions for each method alongside the true distribution for two distinct values of the observation $y$. The left panel corresponds to $y = 1$ and the right panel corresponds to $y = 2$. }
    \label{fig:spherical_fig}
\end{figure}

Consider the model
\[
Y = \sum_{i=1}^{d_\mathcal{X}} X_i^2 + \sigma\varepsilon,
\quad
\varepsilon \sim \mathcal{N}(0,1),
\]
where $\sigma > 0$ controls the observation noise level and $X$ has a standard Gaussian distribution in $\mathbb{R}^{d_\mathcal{X}}$. This example is designed to evaluate the proposed methods in a setting where the posterior concentrates on a nonlinear low-dimensional manifold embedded in a high-dimensional ambient space.

To see this, consider the noiseless case $\sigma = 0$, where the observation equation reduces to $Y = \|X\|_2^2$. For a fixed observation $Y = y > 0$, the states satisfying the constraint $y = \sum_{i=1}^{d_\mathcal{X}} X_i^2$ lie on the $(d_\mathcal{X}-1)$-dimensional sphere of radius $\sqrt{y}$. Hence, although the ambient state space is $\mathbb{R}^{d_\mathcal{X}}$, the conditional distribution of $X$ given $Y = y$ concentrates around a nonlinear manifold of dimension $d_\mathcal{X} - 1$. In the special case $d_\mathcal{X} = 2$, the constraint $y = X_1^2 + X_2^2$ defines a circle in $\mathbb{R}^2$.

The numerical results for $\sigma = 2 \times 10^{-1}$ are presented in Figure~\ref{fig:spherical_fig}. The left panel shows the scatter plot of the samples for each method alongside the ground-truth samples and the dashed black circle of radius $\sqrt{y}$ representing the posterior mean. The right panel depicts the same comparison for $y = 2$. The results demonstrate that CWAE outperforms LREnKF in capturing the posterior mean for $d_{\mathcal{U}} = 1$.

\subsection{Incompressible Flow Field Reconstruction}\label{sec: flow_example}



We evaluate the proposed method in a high-dimensional setting, where the task is to reconstruct velocity fields from sparse, low-dimensional, and noisy observations. The flow field is simulated using the Lattice Boltzmann method for 2D incompressible flow\footnote{\url{https://github.com/mohamedayoub97/Lattice-Boltzmann-Simulation-LBS}} past a cylindrical obstruction at Reynolds number 281, exhibiting Kármán vortex street formation without chaotic behavior, as depicted in Fig.~\ref{fig: flow_field_simulation_samples}. 

We denote the velocity field by $\textbf{u}=(u_x, u_y)$. The state $X \in \mathbb{R}^{d_{\mathcal{X}}}$ is defined as the $d_{\mathcal{X}} = 2\times m \times m$ discretization of the two components of the velocity field. The discretization grid is on a window covering the cylinder wake. The state $X$ is also corrupted with additive process noise $\mathcal{N}(0, \gamma \mathrm{Id}_{d_{\mathcal{X}}})$. The observation function $h: \mathbb{R}^{d_{\mathcal{X}}} \to \mathbb{R}^{d_{\mathcal{Y}}}$ selects $d_{\mathcal{Y}}$ pixels arranged on a uniform grid, as illustrated in Figure~\ref{fig: flow_field_reconstruction}, yielding the observation model
\[
Y = h(X) + \sigma\varepsilon, \qquad \varepsilon \sim \mathcal{N}(0, \mathrm{Id}_{d_{\mathcal{Y}}}).
\]

To encourage the encoder and decoder networks to produce physically meaningful flow fields, we add a divergence and smoothness penalty to our loss~\eqref{eq:ultimate-loss} that penalizes velocity fields that violate incompressibility or are discontinuous ~\cite{RAISSI2019686,10.1007/978-3-319-49409-8_1}. The expressions for the additional penalty terms are evaluated using standard finite-difference techniques and appear in our github repository.

\begin{figure}[t]
    \centering
    \includegraphics[width=1\linewidth, trim={0 0 0 0}]{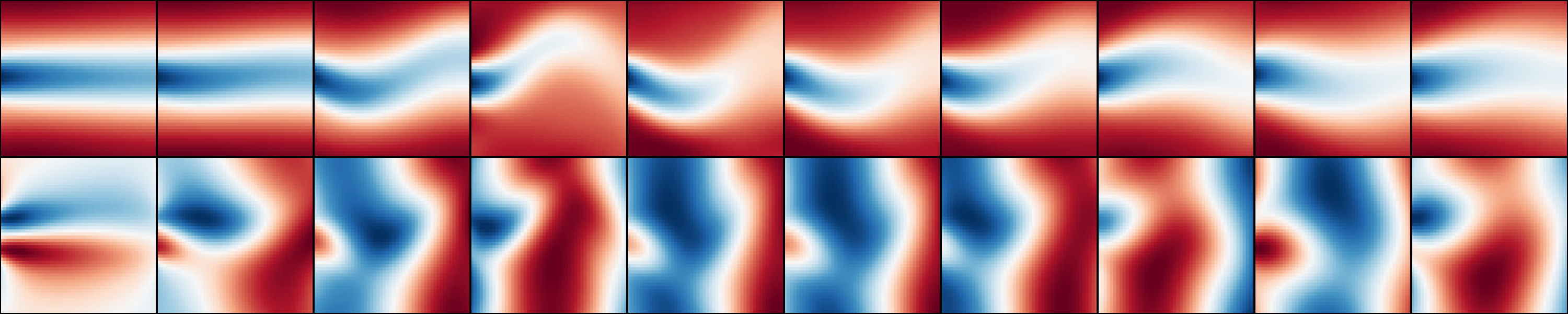}
    \caption{Simulation results for the flow field example in Sec.~\ref{sec: flow_example}. (upper) $u_x$ component over time showing a sinusoidal stream and deceleration. (lower) $u_y$ component over time showing alternating sign and symmetric vortex formation. Taken with $\text{Re}=281$ and a cylindrical obstruction.}
    \label{fig: flow_field_simulation_samples}
\end{figure}

\begin{figure}[ht]
    \centering

    \begin{subfigure}{0.48\linewidth}
        \centering
        \includegraphics[width=1\linewidth, trim={25 25 25 25}]{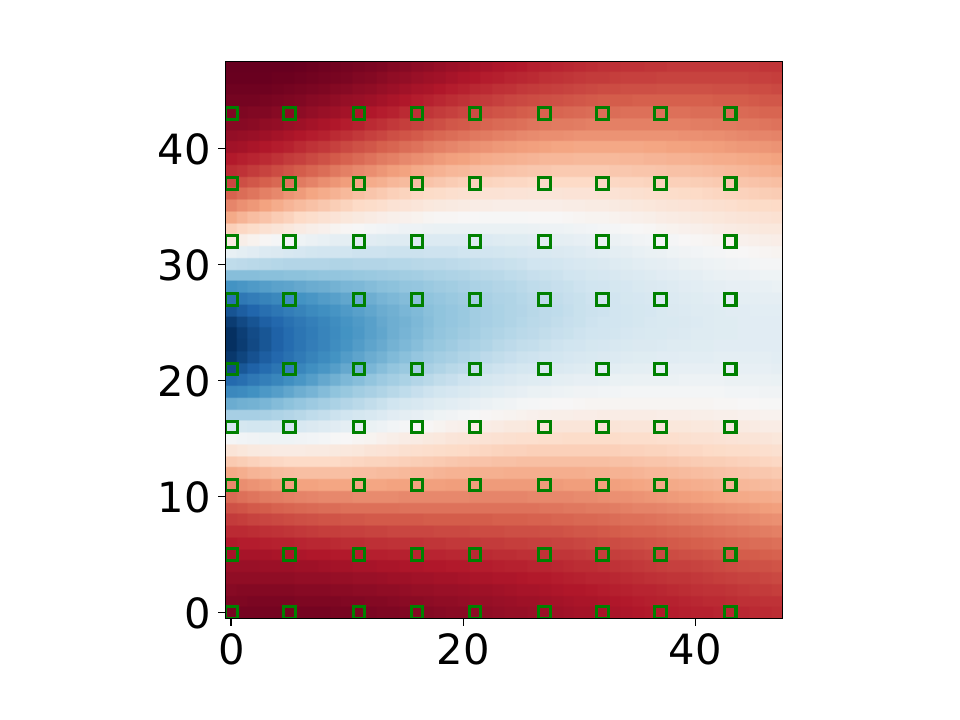}
        \caption{original $u_x$}
        \label{fig:a}
    \end{subfigure}%
    \hfill
    \begin{subfigure}{0.48\linewidth}
        \centering
        \includegraphics[width=1\linewidth, trim={25 25 25 25}]{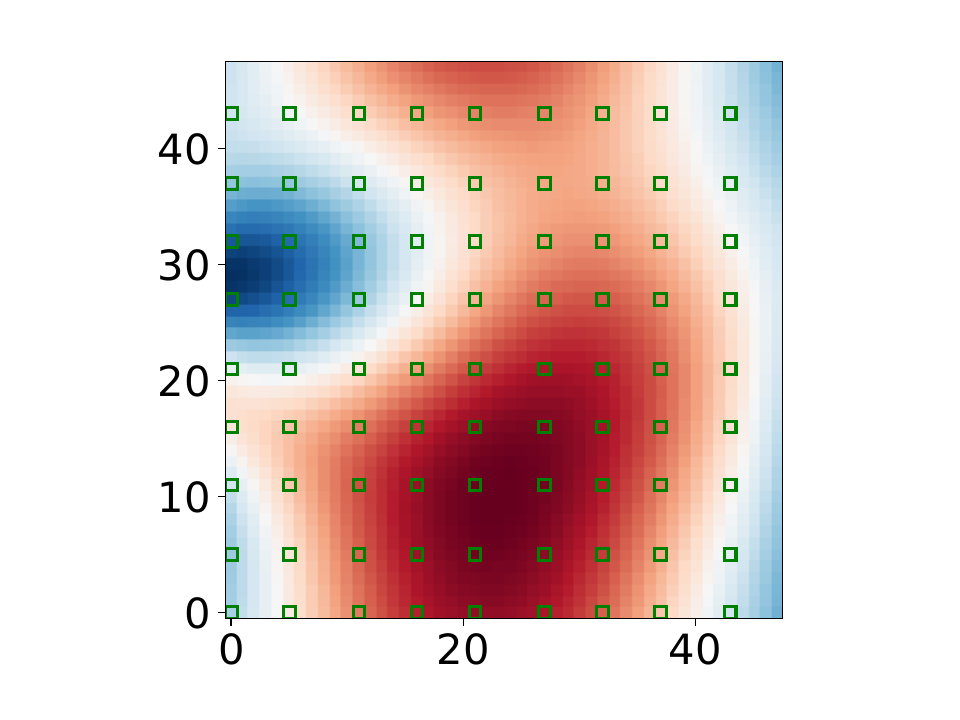}
        \caption{original $u_y$}
        \label{fig:b}
    \end{subfigure}%

    \vspace{0.75em}

    \begin{subfigure}{0.48\linewidth}
        \centering
        \includegraphics[width=1\linewidth, trim={25 25 25 25}]{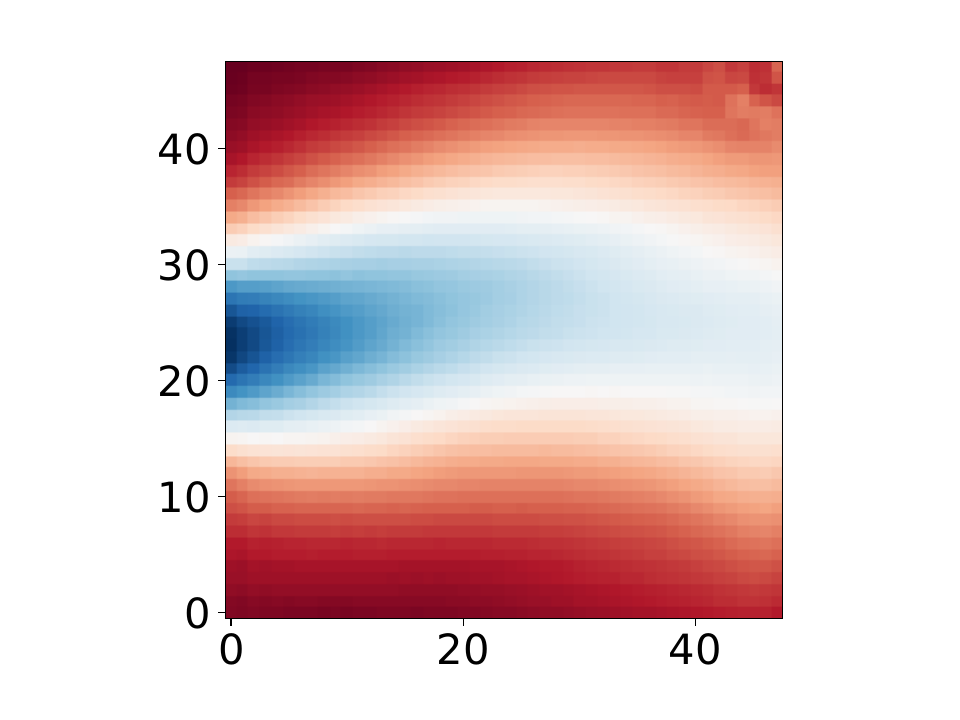}
        \caption{average $\hat{u}_x$}
        \label{fig:c}
    \end{subfigure}%
    \hfill
    \begin{subfigure}{0.48\linewidth}
        \centering
        \includegraphics[width=1\linewidth, trim={25 25 25 25}]{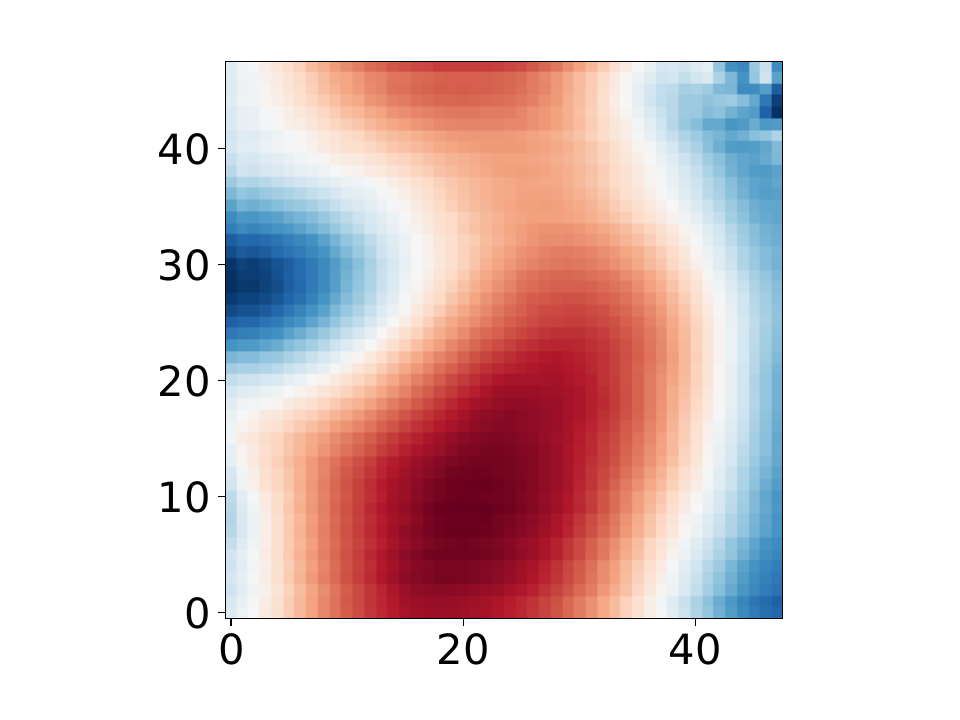}
        \caption{average $\hat{u}_y$}
        \label{fig:d}
    \end{subfigure}%
    \caption{Numerical results for the flow field example in Sec.~\ref{sec: flow_example}. (a) and (b) are the original field with the $d_{\mathcal{Y}}=9 \times 9$ observed pixels boxed in green. The bottom panels are reconstructed from the noisy observation and averaged per-component over 1000 samples.}
    \label{fig: flow_field_reconstruction}
\end{figure}

\begin{figure}[ht]
    \centering

    \begin{subfigure}{0.48\linewidth}
        \centering
        \includegraphics[width=1\linewidth, trim={0 0 0 0}]{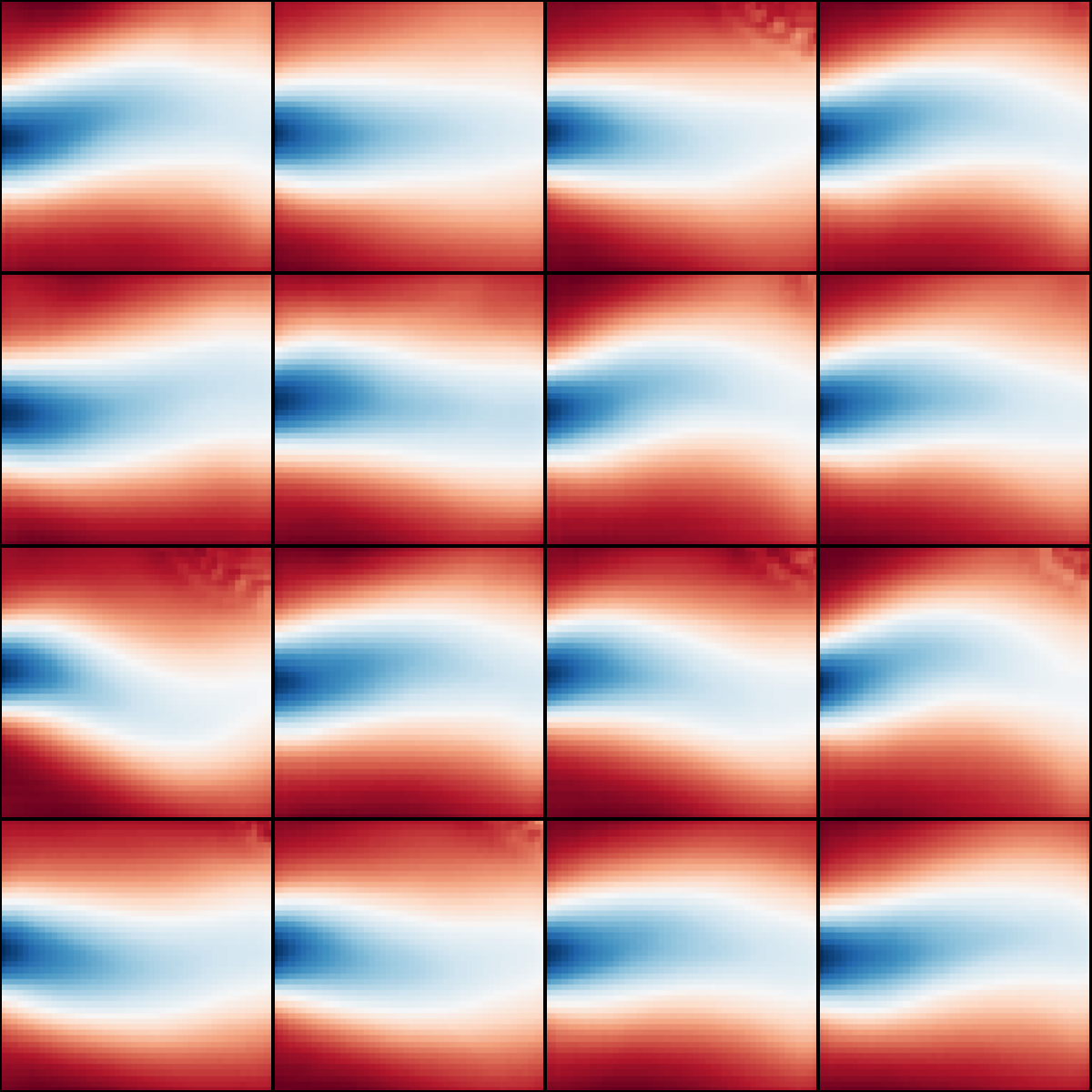}
        \caption{$\hat{u}_x$ samples}
    \end{subfigure}%
    \hfill
    \begin{subfigure}{0.48\linewidth}
        \centering
        \includegraphics[width=1\linewidth, trim={0 0 0 250}]{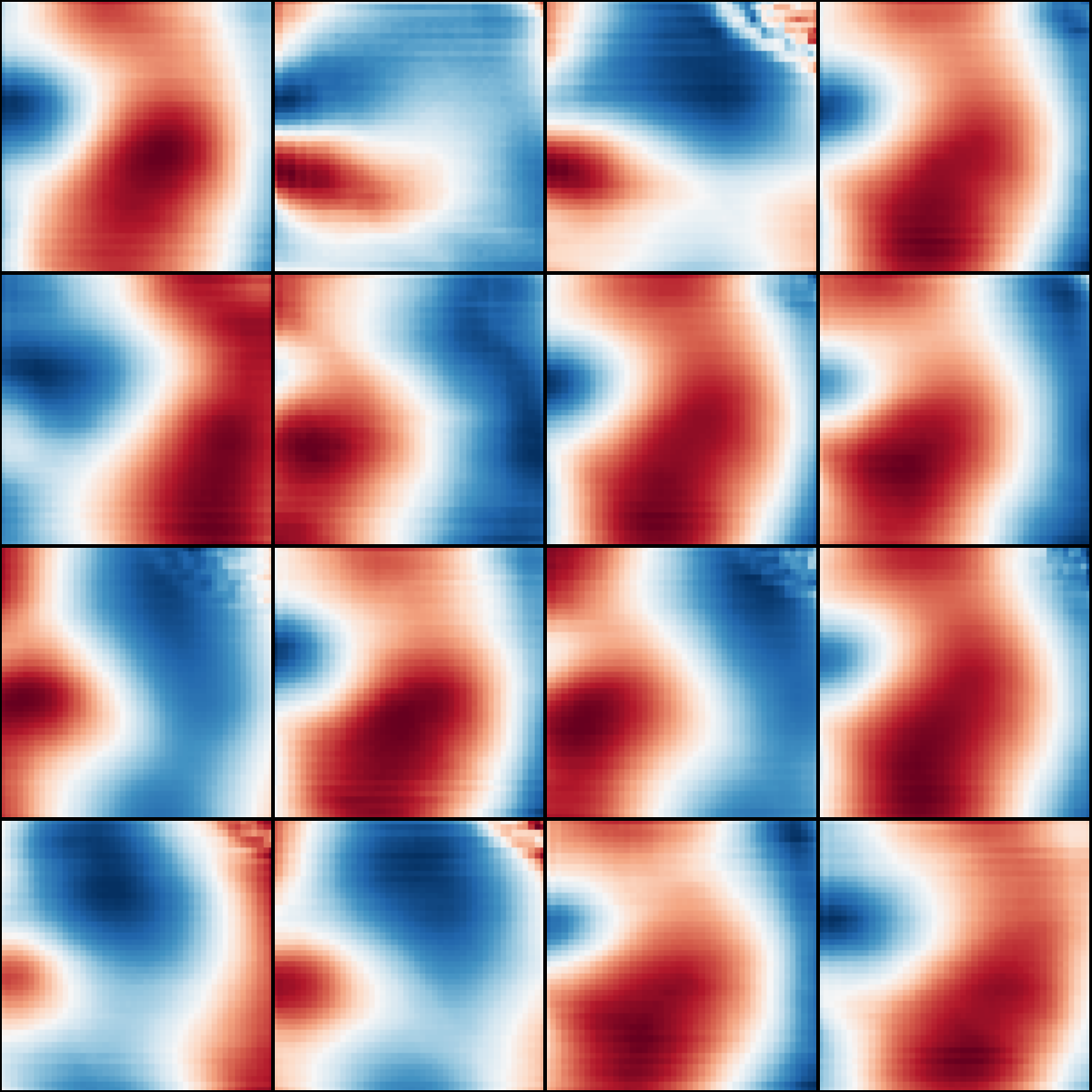}
        \caption{$\hat{u}_y$ samples}
    \end{subfigure}%
    \caption{Numerical results for the flow field example in Sec.~\ref{sec: flow_example} showing 16 out of 1000 samples. The left panel corresponds to the $u_x$ component and the right panel corresponds to the $u_y$ component.}
    \label{fig: flow_field_samples}
\end{figure}

Performance is quantified via the relative mean squared error (MSE) in estimating a function $\phi$ of the state:
\[
\text{MSE}_{rel}(\phi) = \mathbb{E} \frac{\left\|   \frac{1}{N}\sum_{i=1}^{N}\phi(X^i) - \phi(X)\right\|^2} {\|\phi(X)\|^2}
\]
where $X^i$ is generated sample from the conditional distribution, $\phi(X) = X$ and $\phi(X) = X^2$ assess first- and second-moment accuracy, respectively.




\begin{table}
\centering
\caption{Relative MSE across methods for the flow field example in Sec.~\ref{sec: flow_example}.}
\begin{tabular}{ccccc}
\hline
\textbf{$\phi(X)$} & \textbf{LREnKF} & \textbf{CWAE1} & \textbf{CWAE2} & \textbf{CWAE3} \\
\hline
$X$  & 0.0437 & 0.0110 & 0.0281 & \textbf{0.0054} \\
$X^2$ & 0.0716 & 0.0225 & 0.0584 & \textbf{0.0124} \\
\hline
\end{tabular}
\label{tab:MSE_comparison}
\end{table}


Figures~\ref{fig: flow_field_reconstruction} and~\ref{fig: flow_field_samples} present results for CWAE2 with $\gamma = 0$, $\sigma = 1 \times 10^{-2}$, demonstrating that the method accurately captures the structure of the $4608$-dimensional field (two velocity components on a $48\times 48$ grid) using $N = 1000$ particles, recovering multiple flow phases while preserving the correct mean. Table~\ref{tab:MSE_comparison} reports the relative MSE for all methods at $N = 5000$ particles, $\gamma = 1 \times 10^{-1}$, and $\sigma = 2 \times 10^{-1}$, averaged over $10$ independent simulations. The results indicate that CWAE outperforms LREnKF in both first- and second-moment accuracy, demonstrating superior fidelity in representing the true distribution of the flow field.

\section{Discussion}\label{sec:discussion}
    
We introduced CWAEs, a conditional sampling framework that integrates 
triangular measure transport with the Wasserstein autoencoder architecture to exploit 
low-dimensional structure in conditioning problems. Three architectural variants were 
proposed and benchmarked against LREnKF across synthetic examples and a 
high-dimensional flow reconstruction task, with results consistently favoring CWAEs. A current limitation is the sensitivity of the training procedure to the regularization 
parameter $\lambda$ and network hyperparameters, which may restrict applicability under 
constrained computational budgets.

\at{Several directions remain open. The first is a theoretical error analysis of the CWAE objective that accounts for the constraint relaxation, the network approximation, and the finite-sample error. The second is extending CWAEs to sequential filtering settings. The third is a systematic comparison of the three variants, together with benchmarks against LIS methods and paired Wasserstein formulations. Finally, alternative divergences, such as maximum mean discrepancy or Sinkhorn-type penalties, may improve training stability.}

    \bibliographystyle{plain}
    \bibliography{references}

\end{document}